\documentclass[journal]{IEEEtran}

\usepackage{url}
\usepackage{enumerate}
\usepackage{subfigure}
\usepackage{amsfonts,mathrsfs}
\usepackage{amssymb,amsmath,amsthm}
\usepackage{verbatim}
\usepackage{acronym}
\usepackage{algpseudocode,algorithm}
\usepackage{mathtools}
\usepackage{cite}
\usepackage{xcolor}

\usepackage{amssymb,amsmath,amsthm,enumitem}

\newtheorem{theorem}{Theorem}

\newtheorem{definition}{Definition}

\newtheorem{lemma}{Lemma}

\newtheorem{proposition}[theorem]{Proposition}
\newtheorem{remark}{Remark}

\ifCLASSINFOpdf
 
\else
  
\fi

\begin{document}

\title{Optimal Adversarial Policies in the Multiplicative Learning System with a Malicious Expert}

\author{S. Rasoul Etesami, Negar Kiyavash, Vincent Leon, and H. Vincent Poor\thanks{S. Rasoul Etesami and Vincent Leon are with Department of Industrial and Enterprise Systems Engineering, University of Illinois at Urbana-Champaign, Urbana, IL, 61801 email: (etesami1,jliang38)@illinois.edu.} \thanks{Negar Kiyavash is with College of Management of Technology, Ecole polytechnique fédérale de Lausanne (EPFL), Lausanne, Switzerland (email: negar.kiyavash@epfl.ch).}\thanks{H. Vincent Poor is with Department of Electrical Engineering, Princeton University, Princeton, NJ, 08540 (email: poor@princeton.edu).}}

\maketitle

\begin{abstract}
We consider a learning system based on the conventional multiplicative weight (MW) rule that combines experts' advice to predict a sequence of true outcomes. It is assumed that one of the experts is malicious and aims to impose the maximum loss on the system. The loss of the system is naturally defined to be the aggregate absolute difference between the sequence of predicted outcomes and the true outcomes. We consider this problem under both offline and online settings. In the offline setting where the malicious expert must choose its entire sequence of decisions a priori, we show somewhat surprisingly that a simple greedy policy of always reporting false prediction is asymptotically optimal with an approximation ratio of $1+O(\sqrt{\frac{\ln N}{N}})$, where $N$ is the total number of prediction stages. In particular, we describe a policy that closely resembles the structure of the optimal offline policy. For the online setting where the malicious expert can adaptively make its decisions, we show that the optimal online policy can be efficiently computed by solving a dynamic program in $O(N^3)$. Our results provide a new direction for vulnerability assessment of commonly-used learning algorithms to adversarial attacks where the threat is an integral part of the system.
\end{abstract}


\IEEEpeerreviewmaketitle

\section{Introduction}
 
The focus of the vast literature on learning with expert advice is coming up with good prediction rules for the learning system even for the worst possible outcome sequence \cite{vovk1990aggregating,kleinberg2010regret,cesa1997use,haussler1995tight,koolen2013pareto,yu2009dsybil}. However, the proposed algorithms are not designed to be robust against malicious strategic experts. Given the prevalence of machine learning algorithms and as a result, automated decision making in distributed settings in many real-world applications, the effect of malicious experts whose goal is to destroy the performance of the system by injecting false predictions cannot be ignored. In this paper, we address this issue by analyzing the performance of the \emph{multiplicative weighted} (MW) learning algorithm \cite{cesa1997use}, widely used in collaborative filtering, in the presence of malicious experts injecting false recommendations. 

There are many motivating examples for considering the effect of malicious experts in real-world learning systems. To name a few, one can consider movie recommendation systems such as IMDB or Netflix where the system relies on the users' feedback (experts) to rate the quality of the movies. However, the users do not always report the true ratings due to various reasons such as manipulating the outcome of the system toward their preferences \cite{etesami2018influence,yu2009dsybil}. As another example, one can consider sensor fusion in networks where a malicious sensor can attempt to attack the system by injecting false signals and cause the central decision-maker to reach incorrect decisions \cite{truong2018optimal}. Moreover, almost all cases of collaborative filtering or distributed decision making are vulnerable to such internal threats.

In this paper, we study the performance of the MW learning algorithm against adversarial attacks where the adversary's goal is to attack the system without having control over the system's prediction rule. The MW update rule is one of the most commonly used schemes for learning from expert advice \cite{littlestone1994weighted,cesa2006prediction,vovk1990aggregating}, in which after each stage of prediction, when the true outcome is revealed, depending on whether the experts were correct or wrong on that stage, the system punishes or rewards the experts, respectively, by decreasing or increasing their relative weights by a multiplicative factor. Thus, learning with expert advice can be modeled in a multistage sequential decision-making framework where at each stage, the recommendation system combines the predictions of a set of experts about an unknown outcome with the aim of accurately predicting that outcome. 

The problem that we consider here was originally proposed in \cite{truong2013optimal} and subsequently studied in \cite{truong2018optimal}, where it was shown that in the case of \emph{logarithmic} loss function the optimal online policy for the malicious expert is a simple greedy policy. This, however, is not quite surprising as the malicious expert's gain by reporting false predictions substantially dominates his credibility loss due to the logarithmic nature of the loss function. As it was shown using numerical analysis in \cite{truong2013optimal,truong2018optimal}, characterizing the optimal policy for the absolute loss function (which is the more interesting case and commonly used in MW learning systems) is much more challenging due to the strong coupling between the gain in reporting false prediction and the loss in credibility. In this work, we answer this question by showing that the same simple greedy algorithm is \emph{asymptotically} optimal in the offline setting. Moreover, we show that although the optimal online policy can have a complicated structure, it can still be computed efficiently using a reduced-size dynamic program.

The problem that we study in this paper also belongs to the general family of many problems such as target tracking, distributed detection under the byzantine attacks, Sybil attack, and causative attack from the taxonomy of adversarial machine learning where the attacker can modify the data in the training or during the operation in order to degrade the performance of a machine learning algorithm \cite{huang2011adversarial,tran2009sybil,newsome2006paragraph,douceur2002sybil}. Our work is also related to \cite{cover1965behavior,gravin2016towards,gravin2017tight} in which a learner plays against an adversary such that at each step the learner has to choose an expert from a pool of experts to follow while the adversary adaptively sets the gains for the experts with the aim of maximizing the overall regret incurred by the learner. The authors in \cite{gravin2016towards} fully characterize the optimal
online policies for the learner and the adversary in the case of 2 and 3 experts and provide some general insights
into how to design an optimal algorithm for the learner and the adversary for an arbitrary number of experts. However, our work is different from those in the sense that the experts in our setting are themselves malicious and can act strategically. Moreover, the performance guarantee in our setting is in terms of the approximation factor rather than the conventional notion of \emph{regret}.

We consider the problem of learning with a malicious strategic expert under both \emph{offline} and \emph{online} settings. More specifically, we consider a system with two experts; one honest and the other malicious. At each round, the honest expert predicts the true outcome with some accuracy, while the malicious expert strategically provides a prediction with the goal of maximizing the loss incurred by the system. We assume that the adversary knows the true outcome and prediction rule of the learning system. For the offline setting, we assume that the adversary reports his entire sequence of predictions at the beginning of the horizon, while for the online setting, the adversary is allowed to look at the past information up to the current stage and then reports his next prediction. The problem that we address in this paper is two-fold: From the malicious expert's point of view, we are interested in knowing the optimal policy which imposes the maximum loss on the learning system, while from the system designer's point of view we are interested in knowing how the widely-applied MW learning algorithm performs in the presence of a malicious expert.

As one of our main results we show that for the case of absolute loss function, the optimal offline policy can be approximated within a factor $1+O(\sqrt{\frac{\ln N}{N}})$ of the one which reports false predictions at all the stages, where $N$ is the total number of prediction stages. This can be viewed as a counterpart for the conventional regret minimization bounds obtained for the MW update rule. Here it is worth noting that obtaining such an approximation ratio is more challenging than obtaining regret bounds commonly used in expert advice settings. This is because the space of feasible policies is exponentially larger than the set of feasible actions. Therefore, for the offline setting, we approximate the offline \emph{policy} rather than the \emph{action}. One implication of our analysis under the offline setting is that the commonly used MW learning algorithm is not robust with respect to adversarial attacks as a naive malicious expert can impose almost the same loss as an optimally strategic malicious expert. We then extend our results to the online setting and formulate the optimal policy of an online adaptive adversary using a dynamic program (DP). In particular, we show that the number of states of this dynamic program grows only linearly in terms of the number of stages which allows us to compute the optimal online policy for the malicious expert efficiently in $O(N^3)$. 

The paper is organized as follows: In Section \ref{sec:model}, we introduce the model formally and discuss some of its salient properties. In Section \ref{eq:main-approximation}, we provide our main results for the case of offline malicious expert and absolute loss function. In Section \ref{sec:generalization}, we provide an efficient algorithm for computing the optimal online policy for the case of two experts, with an extension for the case of multiple honest experts. Simulation results for offline and online adversaries are provided in Section \ref{sec:simulation}. We conclude the paper by identifying some future directions of research in Section \ref{sec:conclusion}. 

\section{Problem Formulation}\label{sec:model}

In this section, we first introduce the mathematical model formally as in \cite{truong2013optimal}, and then provide some of its salient properties which will be used in our later analysis. In the remainder of this paper, we shall refer to the ill-intent expert as a malicious expert or an adversary, interchangeably.

Consider a learning system with two experts. At each round $k=0,1,2,\ldots$, expert $i\in\{1,2\}$ has a nonnegative weight denoted by $p^i_k\in[0,1]$. We assume that both experts start with equal initial weight $p^1_0=p^2_0=1$. We denote the prediction of the $i$th expert at stage $k$ by $x^i_k\in\{0,1\}$, and the true outcome by $y_k\in\{0,1\}$. At stage $k$, the relative weight of expert $i\in\{1,2\}$ is defined to be 
\begin{align}\label{eq:reletive-weight}
\tilde{p}^i_k:=\frac{p^i_k}{p^1_k+p^2_k}.
\end{align}
In the $k$th stage, the learning system predicts the true outcome $y_k$ using a weighted average rule given by 
\begin{align}\label{eq:estimate-y-hat}
\hat{y}_k=\tilde{p}^1_kx^1_k+\tilde{p}^2_kx^2_k,
\end{align}
and updates the experts' weights in the next time step depending on whether they were correct or wrong in the previous instance using the following multiplicative weight (MW) update rule:
\begin{align}\label{eq:update-rule-original}
p^i_{k+1}&= \begin{cases} p^i_{k}\epsilon & \mbox{if } \  x^i_k\neq y_k, \\
p^i_{k}  & \mbox{if} \  \ x^i_k=y_k.\end{cases}
\end{align}
Here $\epsilon\in (0,1)$ is a fixed constant parameter set the learning system and reflects its aggressiveness on punishing/rewarding the experts. We note that the MW update rule \eqref{eq:update-rule-original} has been extensively used in the past literature  \cite{cesa1997use,cesa2006prediction,auer2002bandit,lattimore2020introbook}. In particular, the MW learning system serves as an independent forecaster (executor). Unlike the adversary, the learning system is neither strategic nor has access to the information of the true outcomes: it merely takes the experts' advice and computes the prediction in each round using \eqref{eq:estimate-y-hat}. After the true outcome $y_k$ is revealed, the system incurs a loss $l(\hat{y}_k,y_k)=Q(|\hat{y}_k-y_k|)$, where $Q(\cdot):[0,1]\to \mathbb{R}^{\ge 0}$ can be some general nondecreasing function. In this paper, we shall only focus on the \emph{absolute} loss function $Q(y):=y$, as it is the most common loss function used in the literature for the expert advice setting \cite{cesa2006prediction,cesa1997use}.

We assume that expert 2 is the {\bf \textit{honest}} expert who makes a correct prediction with accuracy $\mu$, i.e., the one that agrees with the true outcome with probability $\mu$: 
\begin{align}\nonumber
x^2_{k}&= \begin{cases} y_{k} & \mbox{w.p. } \  \mu, \\
1-y_{k}  & \mbox{w.p.} \  1-\mu.\end{cases}
\end{align}

\begin{remark}
For asymmetric accuracies $\{\mu_k\}_{k\in[N]}$, one can partition the horizon into epochs of a small constant length. As predictions of the honest expert are independent, one may assume that the honest expert's expected accuracy within each epoch is close to its expected value denoted by $\mu$. Therefore, our analysis can be viewed as a constant approximation of the heterogeneous model in the stationary regime.
\end{remark}

Expert 1 is the {\bf \textit{malicious}} expert (adversary) who aims to impose the maximum loss on the system by taking the best adversarial action at each stage. We assume that expert 1 knows the true outcome $y_k$ at time $k\in[N]:=\{1,\ldots, N\}$, as well as the distribution of $x^2_k$, the prediction of expert 2.\footnote{Note that the assumption that the adversary knows the prediction accuracy $\mu$ is not very restrictive as the adversary can always learn this distribution using the empirical history of observed actions taken by the honest expert.} One of our main objectives in this paper is to evaluate the robustness of the MW learning algorithm in the presence of a malicious expert. For that reason, we evaluate the system's performance against the most adverse scenario where the adversary has full information about the sequence of outcomes. We refer to Proposition \ref{prop-no-inf} for a weaker adversary with no information about an arbitrary sequence of true outcomes.

\begin{definition}
A malicious expert is called an {\bf \textit{offline}} adversary if he chooses his entire of sequence of predictions $\{x^{1}_k\}_{k=1}^{N}$ at the beginning of the horizon and then commits to it. A malicious expert is called an {\bf \textit{online}} adversary if the entire history of predictions and true outcomes $\{\tilde{p}^1_{\ell}, x^{1}_{\ell}, x^{2}_{\ell}, y_{\ell}\}_{\ell=1}^{k-1}$ are available to him, and then he decides $x^{1}_k$.
\end{definition}

Finally, the goal of the malicious expert (either offline or online) is to produce a sequence of predictions $\{x^{1}_k\}_{k=1}^{N}$ over a fixed finite horizon $N$ in order to maximize the expected aggregate loss on the system given by:
\begin{align}\label{eq:value-function}
\mathbb{E}_{x^2_1,\ldots,x^2_N}[\sum_{k=1}^{N}l(\hat{y}_k,y_k)]=\sum_{k=1}^{N}\mathbb{E}_{x^2_1,\ldots,x^2_k}[l(\hat{y}_k,y_k)],
\end{align}where the second expectation is taken over the past and current actions of the honest agent $x^2_1,\ldots,x^2_k$. In particular, an {\bf \textit{optimal policy}} for the offline/online malicious expert is a sequence of decisions which maximizes the objective function \eqref{eq:value-function} with respect to its corresponding information set, i.e., a solution to the maximization problem $\max_{x^1_1,\ldots,x^1_N}\sum_{k=1}^{N}\mathbb{E}_{x^2_1,\ldots,x^2_k}[l(\hat{y}_k,y_k)].$ 

It is worth noting that one of the major differences between the above model and the conventional expert advice problem is that in the latter one assumes that all the experts are honest and simply report their true recommendations. In particular, the goal is to devise a learning scheme which combines the experts' recommendations in an intelligent manner to accurately predict the unknown outcomes, where it can be shown that the well-known MW learning rule achieves the minimum regret bound. However, the above adversarial model can be viewed as a dual to the expert advice problem where the MW rule is fixed as the underlying learning process and the goal to evaluate how well this learning rule will perform in the presence of a malicious expert who strategically aims to maximize the loss of the system.

\subsection{Preliminary Results}
Here, we describe some of the important properties of the aforementioned adversarial model which will be used later to establish our main results. First we note that using the update rule \eqref{eq:update-rule-original} and the definition of relative weights \eqref{eq:reletive-weight}, we have 
\begin{align}\label{eq:update-rule-relative}
\tilde{p}^1_{k+1}&= \begin{cases} \frac{1}{1+\left(\frac{1}{\tilde{p}^1_{k}}-1\right)\frac{1}{\epsilon}} & \mbox{if } \  x^1_k=1-y_k, x^2_k=y_k, \\
\frac{1}{1+\left(\frac{1}{\tilde{p}^1_{k}}-1\right)\epsilon}   & \mbox{if} \  x^1_k=y_k, x^2_k=1-y_k,\\
\tilde{p}^1_{k} & \mbox{if } \  x^1_k=x^2_k,\end{cases}
\end{align}
In particular, from \eqref{eq:update-rule-relative} one can easily see that the adversary's relative weight changes only when his prediction is at odds with the prediction of the honest agent (when both experts predict the same, the adversary's relative weight remains unchanged). As the update rule in \eqref{eq:update-rule-relative} plays an important role in our analysis, we define a weight update function $g:(0, 1]\!\to\! (0, 1]$ and its inverse $g^{(-1)}:(0, 1]\!\to\! (0, 1]$ by
\begin{align}\label{eq:update-functions}
&g(\rho):=\frac{1}{1+\left(\frac{1}{\rho}-1\right)\frac{1}{\epsilon}},\cr  
&g^{(-1)}(\rho):=\frac{1}{1+\left(\frac{1}{\rho}-1\right)\epsilon}.
\end{align}
In fact, both $g(\rho)$ and its inverse $g^{(-1)}(\rho)$ are strictly increasing functions and we have $g(\rho)\leq \rho\leq g^{-1}(\rho), \forall \rho\in(0,1]$. An important feature of the functions $g(\rho)$ and $g^{(-1)}(\rho)$ is that for any integer $j\in \mathbb{Z}^+$, we have
\begin{align}\label{eq:composition-rule}
g^{(j)}(\rho):&=\underbrace{g(\ldots(g(\rho))}_\text{$j$ times}=\frac{1}{1+\left(\frac{1}{\rho}-1\right)\frac{1}{\epsilon^{j}}},\cr
g^{(-j)}(\rho):&=\underbrace{g^{(-1)}(\ldots(g^{(-1)}(\rho))}_\text{$j$ times}=\frac{1}{1+\left(\frac{1}{\rho}-1\right)\epsilon^j},
\end{align}
where $g^{(j)}(\rho)$ and $g^{(-j)}(\rho)$ denote the composition of $g(\rho)$ and $g^{(-1)}(\rho)$ by themselves $j$ times, respectively. In particular, we note that $g^{(0)}(\rho)\equiv\rho$. 


\section{Optimal Offline Policy for the Absolute Loss Fuction}\label{eq:main-approximation}

In this section, we analyze the optimal policy for the offline adversary and postpone our analysis for the case of the online adversary to Section \ref{sec:generalization}.  We recall that the offline adversary is the one who chooses his entire sequence of decisions (predictions) at the beginning of the horizon. More precisely, the offline adversary aims to maximize the expected loss of the learning system given by \eqref{eq:value-function} over all the $2^N$ feasible sequences of the form $\{0, 1\}^N$. Note that although the space of feasible solutions is exponentially large, however we are only interested in obtaining polynomial-time computable policies.  Therefore, our goal here is to approximate the optimal offline policy within only a negligible additive error term in the overall objective cost.

Toward this end, we first establish a sequence of lemmas to prove our main approximation result (Theorem \ref{thm:main}). In fact, many of these lemmas do not make any use of the specific structure of the functions $g(\rho)$ and $Q(\cdot)$, and we state them in a more general form. Later, in order to provide more closed-form approximation results, we specialize these lemmas to the specific choice of $g(\rho)$ given in \eqref{eq:update-functions} and linear loss function $Q(y)=y$. It is worth noting that although we assumed that the learning algorithm starts with equal initial weight for both experts (i.e., the initial relative weight of the adversary is $0.5$), however, we state our results for an offline adversary with generic initial relative weight $\rho$. The reason for this choice would become apparent subsequently. Next, we state the following lemma from \cite[Lemma 1]{truong2018optimal} whose proof is by induction on the horizon length $N$.

\begin{lemma}\label{lemm:value}
For a loss function $l(\hat{y},y)=Q(|\hat{y}-y|)$, with $Q:[0,1]\to \mathbb{R}^{\ge 0}$, the expected loss given in \eqref{eq:value-function} is fully determined by the initial relative weight of the adversary $\rho$, his policy $\Psi:=(x^1_1,\ldots,x^{1}_N)\in\{0,1\}^
N$, and the horizon length $N$.
\end{lemma}

From Lemma \ref{lemm:value} one can see that the adversary can take his optimal actions by only adjusting them relative to the honest expert's actions. Henceforth, the expected loss in \eqref{eq:value-function} for a given policy $\Psi=(x^1_1,x^1_2,\ldots,x^1_n)$ of the offline adversary can be represented by $V^{\Psi}_n(\rho):=\sum_{k=1}^{n}\mathbb{E}_{x^2_1,\ldots,x^2_k}[l(\hat{y}_k,y_k)]$, where $\rho$ denotes the initial relative weight of the adversary.

\begin{definition}\label{def:false-true}
Assume the adversary's initial weight is $\rho$ and the number of stages is $n$. An adversary's policy is called a {\bf \textit{false policy}} if he lies in all the stages, i.e., $x^1_k=1-y_k, \forall k\in[n]$. It is called a {\bf \textit{true policy}} if the adversary tells the truth in all the stages, i.e., $x^1_k=y_k, \forall k\in[n]$. We let $V^{\rm f}_n(\rho)$ and $V^{\rm t}_n(\rho)$ denote the expected loss of the system if the adversary follows the false policy and the true policy, respectively.
\end{definition}

Using the above definition we can obtain closed-form relations for the expected loss of the false/true policies as given in the following lemma. We will use these expressions as black-boxes in our approximation analysis.  
\begin{lemma}\label{lemm:block}
For a loss function $l(\hat{y},y):=Q(|\hat{y}-y|)$, initial adversary's weight $\rho$, and $n$ stages, we have
\begin{align}\nonumber
&V^{\rm f}_n(\rho)=n(1-\mu)Q(1)+\sum_{j=0}^{n}\mathbb{P}(Z>j)Q(g^{(j)}(\rho)),\cr 
&V^{\rm t}_n(\rho)=n\mu Q(0)+\sum_{j=0}^{n}\mathbb{P}(W>j)Q(1-g^{(-j)}(\rho)),
\end{align}where $Z\sim Bin(n,\mu)$ and $W\sim Bin(n,1-\mu)$ are Binomial distributions with parameters $\mu$ and $1-\mu$, respectively.
\end{lemma}
\begin{proof}
Let us fix the adversary's policy to the false policy, and we look at all the possible sample paths which can be realized by predictions of the honest expert. Any sample-path in which the honest expert predicts correctly $i$ times and makes a mistake $n-i$ times will occur with the same probability of $\mu^i(1-\mu)^{n-i}$. There are exactly ${n \choose i}$ of such sample paths, and for any of such sample paths, independent of what positions the honest expert predicts correctly or wrongly, the incurred loss given the fixed adversary's false policy equals to $(n-i)Q(1)+\sum_{j=0}^{i-1}Q(g^{(j)}(\rho))$. This is because, for any of $n-i$ false predictions of the honest agent in this sample-path, the system incurs a loss of $Q(1)$ and by \eqref{eq:update-rule-relative} the relative weight of the adversary does not change. Moreover, for the remaining $i$ correct predictions and regardless of their order, the system incurs a loss of $\sum_{j=0}^{i-1}Q(g^{(j)}(\rho))$ (note that for $i=0$ this term equals to 0). Therefore, by taking an expectation over all possible sample paths we have,
\begin{align}\nonumber
&V^{\rm f}_n(\rho)=\sum_{i=0}^{n}{n \choose i}\mu^i(1-\mu)^{n-i}(n-i)Q(1)\cr 
&\qquad\qquad+\sum_{i=0}^{n}{n \choose i}\mu^i(1-\mu)^{n-i}\sum_{j=0}^{i-1}Q(g^{(j)}(\rho))\cr 
&=n(1-\mu)Q(1)+\sum_{i=0}^{n}\sum_{j=0}^{i-1}{n \choose i}\mu^i(1-\mu)^{n-i}Q(g^{(j)}(\rho))\cr 
&=n(1-\mu)Q(1)+\sum_{j=0}^{n-1}\sum_{i=j+1}^{n}{n \choose i}\mu^i(1-\mu)^{n-i}Q(g^{(j)}(\rho))\cr 
&=n(1-\mu)Q(1)+\sum_{j=0}^{n}\mathbb{P}(Z>j)Q(g^{(j)}(\rho)),
\end{align}
where in the last equality we have used the fact that $Z\sim Bin(n,\mu)$ and $\mathbb{P}(Z>n)=0$. 

Similarly, to compute $V^{\rm t}_n(\rho)$ we can fix the adversary's policy to the true policy. Now for any sample-path realized by the honest expert with $i$ correct predictions, the system incurs a loss of $i\cdot Q(0)+\sum_{j=0}^{n-i-1}Q(1-g^{(-j)}(\rho))$. Finally, taking an expectation over all sample paths we get
\begin{align}\nonumber
&V^{\rm t}_n(\rho)=\sum_{i=0}^{n}{n \choose i}\mu^i(1-\mu)^{n-i}i\cdot Q(0)\cr
&\qquad\qquad+\sum_{i=0}^{n}{n \choose i}\mu^i(1-\mu)^{n-i}\sum_{j=0}^{n-i-1}Q(1-g^{(-j)}(\rho))\cr 
&=n\mu Q(0)+\sum_{j=0}^{n-1}\sum_{i=0}^{n-j-1}{n \choose i}\mu^i(1-\mu)^{n-i}Q(1-g^{(-j)}(\rho))\cr
&=n\mu Q(0)+\sum_{j=0}^{n}\mathbb{P}(W>j)Q(1-g^{(-j)}(\rho)),
\end{align}
where the second equality is by switching the order of summations, and the last equality is by $W\sim Bin(n,1-\mu)$.
\end{proof}

Next, we note that every offline policy can be partitioned into several blocks such that within each block the adversary follows either false or true policies. Thus we can characterize any offline policy by simply determining the length of each of its sub-blocks. For this purpose, let $n_1,m_1,n_2,m_2,\ldots,n_k,m_k$, denote the partition of the entire horizon $N$ into some sub-horizons of integer length for some positive integer $k$ such that $N=\sum_{i=1}^k(n_i+m_i)$, and $m_i,n_i\in \mathbb{Z}^+$ (note that $n_1$ or $m_k$ can also be zero). We assume that the adversary follows the false policy within each block of length $n_i$, and the true policy within each block of length $m_i$. Therefore, finding the optimal offline policy reduces to maximizing the expected loss \eqref{eq:value-function} over all such partitions.

\begin{lemma}\label{lemm:distribution}
Given an adversary's initial relative weight $\rho$, the relative weight of the adversary after lying $n$ times and telling truth $m$ times (in any arbitrary order) equals to $g^{(X-Y)}(\rho)$, where $X\sim Bin(n,\mu)$ and $Y\sim Bin(m,1-\mu)$ are independent Binomial random variables.
\end{lemma}
\begin{proof}
Let $\bar{X}_i\sim Ber(\mu)$ and $\bar{Y}_i\sim Ber(1-\mu), i=1,2,\ldots$ be independent Bernoulli random variables, and $\rho$ be the initial weight of the adversary. Since at each stage the honest expert predicts independently from the earlier stages,  a simple induction shows that if the adversary's weight at the beginning of the $k$th stage equals to $g^{(U)}(\rho)$ for some random variable $U$, then after the $k$th stage depending on whether he lies or tells the truth, his weight will change to $g^{(U+\bar{X}_k)}(\rho)$ or $g^{(U-\bar{Y}_k)}(\rho)$, respectively. Therefore, if we know that the adversary lies exactly $n$ times and tells the truth $m$ times, his relative weight at the end of this process will be equal to $g^{(X-Y)}(\rho)$, where $X$ is the sum of $n$ independent Bernoulli random variables of type $\bar{X}_i\sim Ber(\mu)$, and $Y$ is the sum of $m$ independent Bernoulli random variables of type $\bar{Y}_i\sim Ber(1-\mu)$. This implies that $X\sim Bin(n,\mu)$ and $Y\sim Bin(m,1-\mu)$, and that $X$ and $Y$ are independent.
\end{proof}

Lemma \ref{lemm:distribution} indicates that the distribution of the adversary's relative weight, induced by following an offline policy $\Psi$, only depends on the total number of times that the adversary lies or tells the truth, and not on the specific order of them. Note that this property only holds for the relative weight distribution, but not for the distribution of the accumulated loss at different stages. In fact, it can be shown that the distribution of loss depends critically on the order of the adversary's actions, and that is the main difficulty in the analysis of the optimal offline policy. We circumvent this issue in the following theorem by providing an approximation scheme that is asymptotically optimal as the number of stages approaches infinity.

\begin{theorem}\label{thm:main}
For any $\epsilon\in (0,1)$, and the absolute loss function $l(\hat{y},y)=|\hat{y}-y|$, we have $\frac{V^{\Psi^*}_N(0.5)}{V^{\rm f}_N(0.5)}=1+O(\sqrt{\frac{\log_{1/\epsilon} N}{N}})$, where $V^{\Psi^*}_N(0.5)$ and $V^{\rm f}_N(0.5)$ denote the expected loss by following the optimal policy and the false policy, respectively.
\end{theorem}
\begin{proof}
For simplicity and without any loss of generality we set $\epsilon=\frac{1}{e}$. For general $\epsilon\in (0,1)$, the only difference in our analysis would be that the base of the natural logarithm will change to $\frac{1}{\epsilon}$. Using Lemma \ref{lemm:block} specialized for the absolute loss function $Q(y)=y$, we can write
\begin{align}\label{eq:greedy}
&V_n^{\rm f}(\rho)=(1-\mu)n+\sum_{j=0}^{n}\mathbb{P}(Z> j)g^{(j)}(\rho),\cr 
&V_n^{\rm t}(\rho)=(1-\mu)n-\sum_{j=0}^{n}\mathbb{P}(W>j)g^{(-j)}(\rho),
\end{align}
where $Z\sim Bin(n, \mu)$ and $W\sim Bin(n,1-\mu)$. For any arbitrary but fixed $\rho\in (0,1)$, let us define 
\begin{align}\nonumber
&f(r):=r-\ln(1+ae^r),\cr 
&\epsilon(r):=f(r+1)-f(r)-g^{(r)}(\rho), 
\end{align}
where $a:=\frac{1}{\rho}-1$ and $r\in [0,\infty)$. Now we can write,
\begin{align}\label{eq:upper-middle-f}
V_n^{\rm f}(\rho)&=(1-\mu)n+\sum_{j=0}^{n}\mathbb{P}(Z> j)[f(j+1)-f(j)-\epsilon(j)]\cr 
&=(1-\mu)n-\sum_{j=0}^{n}\mathbb{P}(Z> j)\epsilon(j)\cr 
&\qquad+\sum_{j=0}^{n}[\mathbb{P}(Z> j-1)-\mathbb{P}(Z> j)]f(j)-f(0)\cr 
&=(1-\mu)n-\sum_{j=0}^{n}\mathbb{P}(Z> j)\epsilon(j)+\mathbb{E}[f(Z)]-f(0)\cr 
&\stackrel{\text{(a)}}{\leq} (1-\mu)n+\mathbb{E}[f(Z)]-f(0)\cr 
&\qquad-\sum_{j=0}^{n}\mathbb{P}(Z> j)\left(\frac{1}{1+ae^{j+1}}-\frac{1}{1+ae^{j}}\right)\cr
&\stackrel{\text{(b)}}{=}(1-\mu)n+\mathbb{E}[f(Z)]-f(0)+g^{(0)}(\rho)-\mathbb{E}[g^{(Z)}(\rho)]\cr 
&\stackrel{\text{(c)}}{=}n-\mathbb{E}[\ln(1+(\frac{1}{\rho}-1)e^Z)]-\ln(\rho)\cr 
&\qquad+\rho-\mathbb{E}[\frac{1}{1+(\frac{1}{\rho}-1)e^Z}],
\end{align}
where $(a)$ is due to Lemma \ref{lemm:appendix} (Appendix A) which shows that $\epsilon(r)\ge \frac{1}{1+ae^{r+1}}-\frac{1}{1+ae^{r}}$, and $(b)$ follows from \eqref{eq:composition-rule} and the definition of expectation. Finally, $(c)$ follows by substituting the expressions for $f(Z)$ and $g^{(Z)}(\rho)$. Similarly, to obtain an upper bound for $V_n^{\rm t}(\rho)$, let us define 
\begin{align}\nonumber
&h(r):=\ln(a+e^r),\cr 
&\delta(r):=h(r+1)-h(r)-g^{(-r)}(\rho).
\end{align}
Using identical steps as in the derivation of \eqref{eq:upper-middle-f} and since by Lemma \ref{lemm:appendix}, $\delta(r)\leq \frac{1}{1+ae^{-(r+1)}}-\frac{1}{1+ae^{-r}}$, we get
\begin{align}\label{eq:upper-h}
V_n^{\rm t}(\rho)&\leq(1-\mu)n-\mathbb{E}[\ln((\frac{1}{\rho}-1)+e^W)]-\ln(\rho)\cr 
&\qquad-\rho+\mathbb{E}[\frac{1}{1+(\frac{1}{\rho}-1)e^{-W}}].
\end{align}

Next let us consider an arbitrary offline policy $\Psi$ characterized by its false/true sub-block, i.e., $\Psi:=n_1,m_1,\ldots,n_k,m_k$. Denote the expected loss under policy $\Psi$ when the initial weight of the adversary was $0.5$ by $V^{\Psi}_N(0.5)$. Moreover, for $\ell=1,\ldots,k$, let $\bar{X}_{\ell}\sim Bin(n_{\ell}, \mu)$ and  $\bar{Y}_{\ell}\sim Bin(m_{\ell}, 1-\mu)$ be pairwise independent Binomial distributions (i.e., for every $i$ and every $j$, $\bar{X}_i$ and $\bar{Y}_j$ are independent) and define $X_{\ell}=\sum_{i=1}^{\ell}\bar{X}_{i}$, and $Y_{\ell}=\sum_{i=1}^{\ell}\bar{Y}_{i}$. Note that the pair $X_{\ell}$ and $Y_{\ell}$ are independent Binomial distributions. By linearity of expectation, and using Lemma \ref{lemm:distribution} we can write,
\begin{align}\label{eq:policy-pi-block-multiplicative}
&V^{\Psi}_N(0.5)=V_{n_1}^{\rm f}(0.5)+\mathbb{E}\left[V_{m_1}^{\rm t}\!\left(g^{(X_1)}(0.5)\right)\!\right]\cr 
&\qquad+\mathbb{E}\left[V_{n_2}^{\rm f}\!\left(g^{(X_1\!-\!Y_1)}(0.5)\right)\!\right]+\mathbb{E}\left[V_{m_2}^{\rm t}\left(g^{(X_2-Y_1)}(0.5)\right)\right]\cr 
&\qquad+\ldots+\mathbb{E}\left[V_{m_k}^{\rm t}\!\left(g^{(X_k\!-\!Y_{k-1})}(0.5)\right)\!\right].
\end{align}
Replacing $g^{(X_{\ell-1}-Y_{\ell-1})}(0.5)$ (or for brevity $g^{(X_{\ell-1}-Y_{\ell-1})}$) from \eqref{eq:composition-rule} instead of $\rho$ in \eqref{eq:upper-middle-f}, and taking expectation we have
\begin{align}\label{n-l-block-upper-bound}
&\mathbb{E}\left[V_{n_{\ell}}^{\rm f}\left(g^{(X_{\ell-1}-Y_{\ell-1})}\right)\right] \cr 
&\qquad\leq\mathbb{E}\left[n_{\ell}-\mathbb{E}\Big[\ln\left(1+(\frac{1}{g^{(X_{\ell-1}-Y_{\ell-1})}}-1)e^{Z}\right)\Big]\right]\cr 
&\qquad\qquad-\mathbb{E}\left[\ln\left(g^{(X_{\ell-1}-Y_{\ell-1})}\right)\right]\cr 
&\qquad\qquad+\mathbb{E}\Big[g^{(X_{\ell\!-\!1}\!-\!Y_{\ell\!-\!1})}-\mathbb{E}\Big[\frac{1}{1\!+\!(\frac{1}{g^{(X_{\ell\!-\!1}\!-\!Y_{\ell\!-\!1})}}\!-\!1)e^{Z}}\!\Big]\Big]\cr  
&\qquad=n_{\ell}-\mathbb{E}\left[\ln\left(\frac{1+e^{X_{\ell}-Y_{\ell-1}}}{1+e^{X_{\ell-1}-Y_{\ell-1}}}\right)\right]\cr 
&\qquad\qquad+\mathbb{E}\Big[\frac{1}{1+e^{X_{\ell-1}-Y_{\ell-1}}}-\frac{1}{1+e^{X_{\ell}-Y_{\ell-1}}}\Big], 
\end{align}
where the equality follows by simplifying the terms and noting that for the $\ell$-the false block $Z:=\bar{X}_{\ell}\sim Bin(n_{\ell},\mu)$, which is independent of $X_{\ell-1}$ and $Y_{\ell-1}$. Similarly, since for the $\ell$-th true block $W:=\bar{Y}_{\ell}\sim Bin(m_{\ell},1-\mu)$, which is independent of $X_{\ell}$ and $Y_{\ell-1}$, by replacing $g^{X_{\ell}-Y_{\ell-1}}(0.5)$ instead of $\rho$ into \eqref{eq:upper-h} and taking expectation we get
\begin{align}\label{m-l-block-lower-bound}
&\mathbb{E}\left[V_{m_{\ell}}^{\rm t}\left(g^{(X_{\ell}-Y_{\ell-1})}(0.5)\right)\right]\cr 
&\qquad\qquad\leq(1-\mu)m_{\ell}-\mathbb{E}\Big[\ln\Big(\frac{e^{X_{\ell}-Y_{\ell\!-\!1}}+\!e^{\bar{Y}_{\ell}}}{1\!+\!e^{X_{\ell}-Y_{\ell\!-\!1}}}\Big)\Big]\cr 
&\qquad\qquad\qquad+\mathbb{E}\Big[\frac{1}{1+e^{X_{\ell}-Y_{\ell}}}\!-\!\frac{1}{1+e^{X_{\ell}-Y_{\ell\!-\!1}}}\Big].
\end{align}Finally, substituting \eqref{n-l-block-upper-bound} and \eqref{m-l-block-lower-bound} into \eqref{eq:policy-pi-block-multiplicative}, we can write
\begin{align}\label{eq:upper-sum}
&V^{\Psi}_N(0.5)\leq \sum_{\ell=1}^{k}(n_{\ell}+(1-\mu)m_{\ell})\cr 
&\qquad-\mathbb{E}\Big[\sum_{\ell=1}^{k}\Big(\!\ln\!\Big(\!\frac{e^{X_{\ell}-Y_{\ell-1}}+e^{\bar{Y}_{\ell}}}{1+e^{X_{\ell}-Y_{\ell-1}}}\Big)\!+\!\ln\Big(\frac{1+e^{X_{\ell}-Y_{\ell-1}}}{1+e^{X_{\ell-1}-Y_{\ell-1}}}\!\Big)\Big)\!\Big]\cr
&\qquad+2\sum_{\ell=1}^{k}\mathbb{E}\Big[\frac{1}{1\!+\!e^{X_{\ell}-Y_{\ell}}}\!-\!\frac{1}{1\!+\!e^{X_{\ell}-Y_{\ell\!-\!1}}}\Big]\cr  
&\qquad=\mu M+(1-\mu)N\!-\!\mathbb{E}\left[\ln\left(\prod_{\ell=1}^{k}\frac{e^{X_{\ell}-Y_{\ell-1}}+e^{\bar{Y}_{\ell}}}{1+e^{X_{\ell-1}-Y_{\ell-1}}}\right)\right]\cr 
&\qquad\qquad+O(\sqrt{N\ln N})\cr
&\qquad= (1\!-\!\mu)N\!-\!\mathbb{E}\!\left[\ln\!\left(\!\frac{1\!+\!e^{Y_k-X_{k}}}{2}\!\right)\!\right]+O(\sqrt{N\ln N})\cr 
&\qquad=(1-\mu)N+O(\sqrt{N\ln N}),
\end{align}
where the first equality is due to Lemma \ref{lemm:main} (Appendix A) and noting that $\sum_{\ell=1}^{k}n_{\ell}=M$ and $\sum_{\ell=1}^{k}m_{\ell}=N-M$. Moreover, the second equality holds because $\bar{Y}_{\ell}=Y_{\ell}-Y_{\ell-1}$ (which causes telescopic cancellation of the product terms inside of the natural logarithm) and noting that  $X_0=Y_0=0$, $\mathbb{E}[X_k]=\mu M$. Using \eqref{eq:greedy}, the expected loss of the false policy is at least, 
\begin{align}\nonumber
V_N^{\rm f}(0.5)=(1-\mu)N+\sum_{j=0}^{N}\mathbb{P}(Z\!>\! j)g^{(j)}(0.5)\ge(1-\mu)N.
\end{align}
(Note that all the terms $\mathbb{P}(Z> j)g^{(j)}(0.5)$ are nonnegative.) This in view of \eqref{eq:upper-sum} shows that $\frac{V^{\Psi^*}_N(0.5)}{V^{\rm f}_N(0.5)}=1+O(\sqrt{\frac{\ln N}{N}})$.
\end{proof}

\begin{remark}
It is important to distinguish the difference between the approximation ratio obtained in Theorem \ref{thm:main} and the sub-linear regret bounds commonly derived in regret minimization analysis. Here we allow the offline malicious expert to choose his offline {\bf policy} over the entire horizon (i.e., an arbitrary sequence of false/true predictions) and do not restrict him to his  {\bf action} (i.e., only select one action and commit to it at all the stages). Surprisingly, Theorem \ref{thm:main} shows that giving such an extra power to the malicious expert does not bring him much gain other than a negligible additive term.
\end{remark}


Finally, we note that Theorem \ref{thm:main} suggests that the MW learning algorithm with absolute loss function is not very robust against malicious attacks. This is because the malicious expert does not to be very intelligent or have access to many computational resources to destroy the outcomes of the MW learning system; even following a simple false policy can nearly impose the same loss as any other complex policy.


\subsection{Beyond Asymptotic Optimality for the Offline Policy}\label{sec:optimal-pattern}

Theorem \ref{thm:main} shows that the false policy asymptotically achieves the same performance as the optimal offline policy. However, the exact structure of the optimal offline policy for a finite horizon $N$ can be quite complex. Therefore, our goal in this section is to take one step further and provide a policy that closely resembles the structural patterns of the optimal offline policy. As it was shown in the proof of Theorem \ref{thm:main} one of the main reasons that there is a gap between the expected loss of the optimal offline policy and that of the false policy is the term: 
\begin{align}\nonumber
B:=\sum_{\ell=1}^{k}\mathbb{E}\left[\frac{1}{1\!+\!e^{X_{\ell}-Y_{\ell}}}-\frac{1}{1\!+\!e^{X_{\ell}-Y_{\ell\!-\!1}}}\right]
\end{align}
henceforth referred to as the {\bf \textit{bonus}} term. Here, $X_{\ell}\sim Bin(N_{\ell},\mu)$ and $Y_{\ell}\sim Bin(M_{\ell},1-\mu)$ are independent Binomial distributions where $N_{\ell}=\sum_{i=1}^{\ell}n_i$ and $M_{\ell}=\sum_{i=1}^{\ell}m_i$. Therefore, it seems reasonable to expect the optimal offline policy (or a policy close to optimal), to maximize $B$ in order to gain as much as possible from the bonus term. As such, we search the optimal offline policy $\Psi^*$ among policies $\Psi$ that satisfy the following two criteria: 
\begin{itemize}
\item $i$) $\Psi$ imposes at least as much loss as the false policy on the learning system, i.e., at least $(1-\mu)N-o(1)$,
\item $ii$) $\Psi$ maximizes the bonus gain $B$. 
\end{itemize}

To maximize the bonus term, using Lemma \ref{lemm:main} (Appendix A), it is enough to maximize
\begin{align}\label{eq:almost-bonus}
\sum_{\ell=1}^{k}\Big[\Phi(-\frac{\mu_{2\ell}}{\sigma_{2\ell}})-\Phi(-\frac{\mu_{2\ell-1}}{\sigma_{2\ell-1}})\Big],
\end{align}
where $\Phi(\cdot)$ is the CDF of the standard normal distribution and  
\begin{align}\nonumber
&\mu_{2\ell}\!=\!N_{\ell}\mu\!-\!M_{\ell}(1-\mu), \ \ \ \ \ \ \ \ \ \sigma^2_{2\ell}=\mu(1-\mu)(N_{\ell}+M_{\ell}),\cr 
&\mu_{2\ell-1}\!=\!N_{\ell}\mu\!-\!M_{\ell-1}(1-\mu), \  \ \ \sigma^2_{2\ell-1}\!=\!\mu(1\!-\!\mu)(N_{\ell}\!+\!M_{\ell\!-\!1}). 
\end{align}  

Now by adjusting the argument of $\Phi(\cdot)$ in \eqref{eq:almost-bonus} to periodically switch around $0$ (see Figure \ref{fig:normal}), we obtain a positive gain from each of the summands in \eqref{eq:almost-bonus}. This suggests that a policy in which $-\frac{\mu_{2\ell}}{\sigma_{2\ell}}= 1$, and $-\frac{\mu_{2\ell-1}}{\sigma_{2\ell-1}}= -1$, would be a good candidate for maximizing the bonus term. Note that here the choice of $1$ or $-1$ is not strict and it can be replaced by any two points close to zero such that the difference of the normal CDF evaluated at those points gives a sufficiently large gain. Solving $-\frac{\mu_{2\ell}}{\sigma_{2\ell}}= 1$ and $-\frac{\mu_{2\ell-1}}{\sigma_{2\ell-1}}= -1$, by substituting the above expressions for $\mu_{2\ell}, \sigma_{2\ell}, \sigma_{2\ell}, \sigma_{2\ell-1}$, we obtain a ratio type policy in which the ratio of the false/true block lengths $\frac{n_{\ell}}{m_{\ell}}, \ell=1,2,\ldots$, is proportional to $\frac{1-\mu}{\mu}$. Therefore, to fulfill both criteria $(i)$ and $(ii)$, we introduce the following offline {\bf\textit{ratio}} policy:
\begin{figure}[t]
\vspace{-1cm}
\begin{center}
\includegraphics[totalheight=.24\textheight,
width=.4\textwidth,viewport=50 0 550 500]{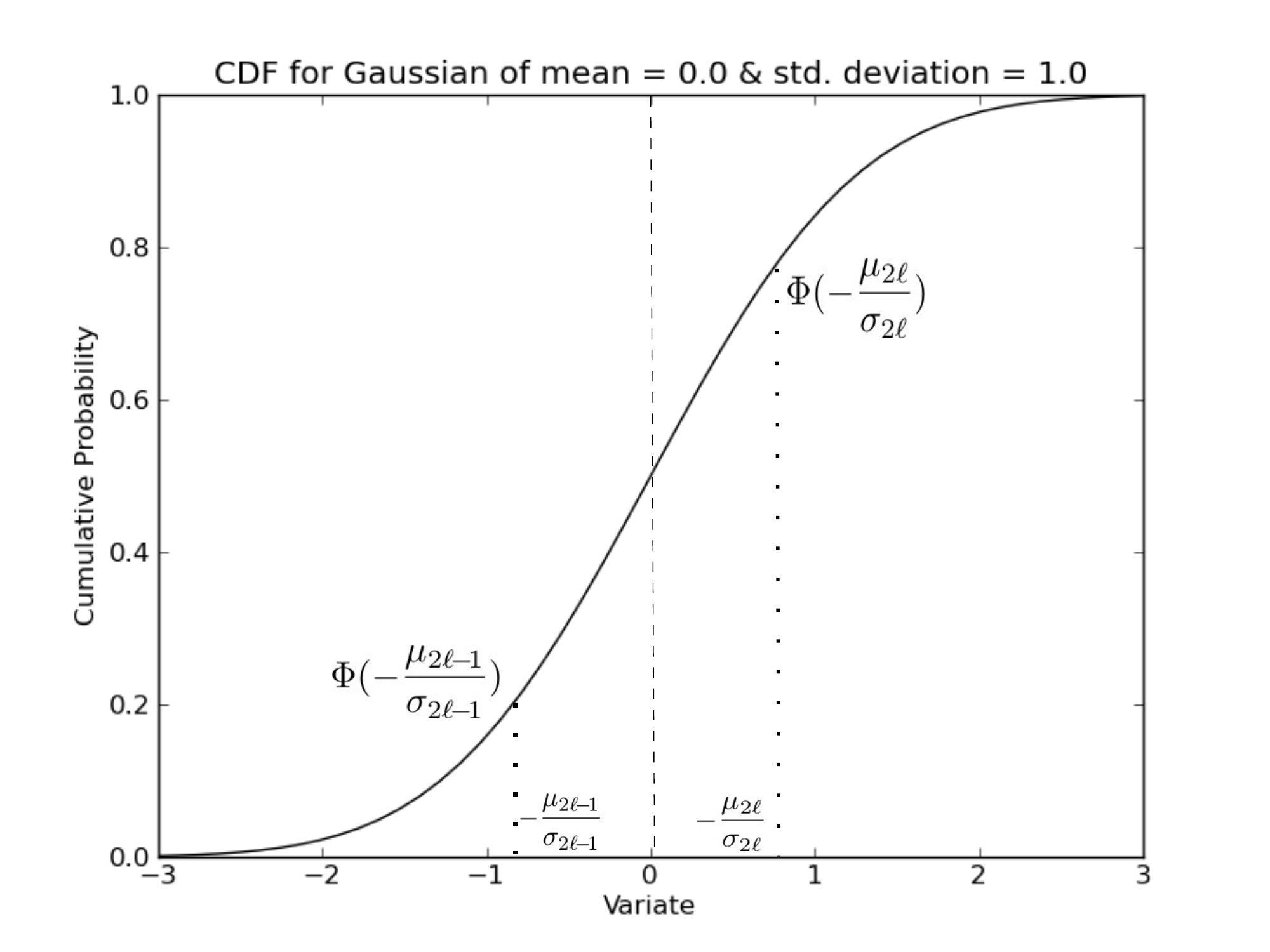} \hspace{0.4in}
\end{center}
\caption{Mean and standard deviation adjustment for maximizing the bonus term $B$ using CDF of the standard normal distribution.}
\label{fig:normal}
\end{figure}

\begin{definition}\label{def:ratio}
Let $a$ and $b$ be the smallest positive integers such that $\frac{a}{b}=\frac{\mu}{1-\mu}$.\footnote{Here for simplicity we have assumed that $\frac{\mu}{1-\mu}$ is a rational number, otherwise, we can always find positive integers such that $\frac{a}{b}\cong\frac{\mu}{1-\mu}$.} We say that $\hat{\Psi}$ is a ratio policy if its block representation is of the form $(n_{1},m_1,\ldots,m_{k-1},n_{k},m_k)=(b,a,b,a,b,\ldots,a,\frac{N}{2},0)$,
where for $\ell=1,2,\ldots,k$, the adversary lies in $n_{\ell}$-blocks and tells the truth in $m_{\ell}$-blocks. 
\end{definition} 

To verify why the ratio policy is indeed a good offline policy, let us denote the number of lies and truths in $\hat{\Psi}$ by $M$ and $N-M$, respectively. Due to Definition \ref{def:ratio}, $\hat{\Psi}$ has many more lies than the truth in its structure (note that $n_{k}=\frac{N}{2}$). Thus a similar analysis as in Lemma \ref{lemm:block} for the false policy reveals that the expected loss in $\hat{\Psi}$ which is obtained from its last block $n_{k}=\frac{N}{2}$ is almost the same as the false policy minus a negligible constant which does not depend on $N$. As a result, the expected loss of $\hat{\Psi}$ which is obtained due to its heavy tail of false predictions is at least $(1-\mu)N-o(1)$. On the other hand, the ratio policy $\hat{\Psi}$ gains extra bonus due to its first $\frac{N}{2}$ stages. To evaluate the bonus term $B$ for the ration policy $\hat{\Psi}$, we observe that due to Definition \ref{def:ratio}, $\mu_{2\ell}=0$, and $\mu_{2\ell-1}=\mu b$, for all $\ell\in [k]$. Thus, for each $\ell\in[k]$, we have 
\begin{align}\nonumber
\Phi(-\frac{\mu_{2\ell}}{\sigma_{2\ell}})-\Phi(-\frac{\mu_{2\ell-1}}{\sigma_{2\ell-1}})=\Phi(0)-\Phi(-\frac{\mu b}{\sigma_{2\ell-1}})>0.
\end{align}
In other words, each of these terms contributes positively with some constant amount to the bonus $B$. Since we chose $a$ and $b$ as the smallest positive integers such that $\frac{a}{b}=\frac{\mu}{1-\mu}$, this assures that the number of summands $k$ in the bonus term $B$ (which equals the number of switching between false/true blocks) is maximized. This is exactly why we defined our ratio policy the way we did in Definition \ref{def:ratio}. This shows that the expected loss of the ratio policy is at least as high as that for the false policy (satisfying criterion $(i)$) with an additional bonus term $B$ (satisfying criterion $(ii)$). 

\section{Optimal Online Policy for the Absolute Loss Function}\label{sec:generalization}

In this section, we consider the problem of finding the optimal online policy for the malicious expert, where we recall that the online adversary is the one who chooses his next action adaptively based on all the past revealed information up to the current stage. In order to be able to find the optimal online policy we first cast it as a dynamic program and then show that it can be solved efficiently in $O(N^3)$.   

For this purpose, let us assume again that the malicious expert is expert $1$ and the other expert is the honest one who makes a correct prediction with probability $\mu$. We assume that at stage $k$, expert $1$ knows the true outcome $y_k$, the accuracy $\mu$ of the honest expert, and the entire history of predictions up to stage $k-1$, i.e., $\{\tilde{p}^1_{\ell},\tilde{p}^2_{\ell}, x^{1}_{\ell},x^{2}_{\ell}, y_{\ell}:  \forall \ell\in[k-1]\}$. Given this information set, the goal of the online malicious expert is to produce a sequence of predictions $\{x^{1}_k\}_{k=1}^{N}$ over a fixed finite horizon $N$ to maximize the expected accumulated loss of the system given by \eqref{eq:value-function}. Now let us define the \emph{state} of the system at stage $k$ to be the relative weight of the adversary at that stage, i.e., $\tilde{p}^1_{k}$ . Note that as $\tilde{p}^1_{k}+\tilde{p}^2_{k}=1, \forall k$, knowing $\tilde{p}^1_{k}$ is sufficient to determine the relative weight of the honest expert $\tilde{p}^2_{k}$.


Next let us define $c_{x^1_k}(\tilde{p}^1_{k})$ to be the \emph{current} loss that the online adversary can impose on the system at stage $k$ by taking the action $x^1_k$, i.e., 
\begin{align}\label{eq:current-cost}
c_{x^1_k}(\tilde{p}^1_{k})\!=\!\mathbb{E}_{x^2_k}[|\hat{y}_{k}\!-\!y_{k}||x^1_1]\!=\!\begin{cases}1\!-\!\mu\!+\!\mu \tilde{p}^1_{k} & \!\! \mbox{if} \ x^1_k=1\!-\!y_k, \\
(1\!-\!\mu) (1\!-\!\tilde{p}^1_{k}) & \!\!\mbox{if} \ x^1_k=y_k, \end{cases}  
\end{align}
where the second equality is by Lemma \ref{lemm:value} specialized to the absolute loss function $Q(y)=y$. We can then cast the adversary's online optimal policy as a solution to an MDP in which the malicious expert's action at stage $k$ imposes a current loss of $c_{x^1_k}(\tilde{p}^1_{k})$ on the system and changes the state from $\tilde{p}^1_{k}$ to the next state $\tilde{p}^1_{k+1}$. In particular, the state transition of this MDP is given by the update rule \eqref{eq:update-rule-relative}, that is,
\begin{align}\label{eq:next-state}
\tilde{p}^1_{k+1}= \begin{cases} g(\tilde{p}^1_{k}) & \mbox{if } \  x^1_k=1-y_k, x^2_k=y_k, \\
g^{(-1)}(\tilde{p}^1_{k})   & \mbox{if} \  \ x^1_k=y_k, x^2_k=1-y_k,\\
\tilde{p}^1_{k} & \mbox{if } \  x^1_k=x^2_k.\end{cases}
\end{align} 

Now the solution to this MDP can be obtained using dynamic programming as shown in Algorithm \ref{eq:alg}. In this algorithm $V^*_{k+1}(\cdot)$ denotes the optimal value function, i.e., the optimally accumulated loss from time step $k+1$ onward. In particular, from Lemma \ref{lemm:value}, one can easily see that the optimal value function does not depend on the sequence of true outcomes and is only a function of the state and the number of remaining stages. Now by substituting the closed form expressions of the current cost \eqref{eq:current-cost} and the state transition \eqref{eq:next-state} into the DP Algorithm \ref{eq:alg}, and letting $\rho:=\tilde{p}^1_{k}$ for brevity, we obtain the following closed-form recursion for computing the optimal value function:
\begin{align}\label{eq:recursion-DP}
V^*_{k}(\rho)\!&=\!\max\Big\{1\!-\!\mu+\mu \rho+\mu V^*_{k+1}(g(\rho))+(1\!-\!\mu)V^*_{k+1}(\rho),\cr 
&\ \ \ \ (1\!-\!\mu) (1\!-\!\rho)+(1\!-\!\mu)V^*(g^{(-1)}(\rho))+\mu V^*_{k+1}(\rho)\Big\},
\end{align}
where the first term in the maximization \eqref{eq:recursion-DP} corresponds to the adversary's action at stage $k$ being $x^1_k=1-y_k$, and the second term corresponds to the adversary's action being $x^1_k=y_k$. Unfortunately, due to the nonlinear structure of the transition functions $g(\rho)$ and $g^{(-1)}(\rho)$, as well as their joint convex/concave structure, solving the recursion \eqref{eq:recursion-DP} in a closed-form seems to be a tedious task. Although at each stage of the above recursion one needs to consider the maximum of two alternatives (so that the number of alternatives will grow exponentially in terms of the number of stages), however, in the following theorem we show that most of these alternatives collapse on each other so that the optimal value function in \eqref{eq:recursion-DP} can be computed efficiently in polynomial time. 


\begin{algorithm}[t]
\caption{DP Algorithm}\label{eq:alg}
 {\bfseries Initialize:} $V_N(\cdot) = c_N(\cdot) = 0$. 

For each step $k = N-1 \mbox{ downto } 0$, find the optimal action
\begin{align}\nonumber
x^{*}_k:=\arg\max_{x^1_k} \big\{c_{x^1_k}(\tilde{p}^1_{k})+\mathbb{E}[V^*_{k+1}(\tilde{p}^1_{k+1})]\big\},
\end{align} 
and the optimal value function,
\begin{small}
\begin{align}\label{eq:5}
V^*_k(\tilde{p}^1_{k})=\max_{x^1_k} \big\{c_{x^1_k}(\tilde{p}^1_{k})+\mathbb{E}[V^*_{k+1}(\tilde{p}^1_{k+1})]\big\}. 
\end{align}\end{small}
 {\bfseries Output:} sequence $x^*_{N-1},V^*_{N-1}(\cdot),...,x^*_{0},V^*_{0}(\cdot).$
\end{algorithm}

\begin{theorem}\label{thm:online-two}
The optimal policy for the online malicious expert can be found in $O(N^3)$, where $N$ is the number of stages. 
\end{theorem}
\begin{proof}
Let us consider a decision tree with a root node representing the initial relative weight of the adversary (i.e., $\rho=0.5$) and such that the
nodes in the $k$-th level of the tree that are at distance $k$ from the root represent all the possible relative weights of the adversary after $k$
stages (Figure \ref{fig:tree}). The key observation is that due to the property of $g(\cdot)$ and its inverse $g^{(-1)}(\cdot)$, the size of this decision tree does not grow exponentially such as a binary tree. In fact, a simple induction shows that the nodes in the $k$-th level of the tree can be grouped to form exactly $2k-1$ nodes representing all possible relative
weights of the adversary up to stage $k$, given by $g^{(-k)}(0.5), g^{(-k+1)}(0.5),\ldots, g^{(k-1)}(0.5), g^{(k)}(0.5)$.\footnote{Using Lemma \ref{lemm:distribution}, one can even compute the distribution of the weights on the reduced nodes efficiently using Binomial distributions.} Therefore, the total number of tree nodes by such grouping (states in the DP) after $N$ stages is at most $\sum_{k=1}^{N}(2k-1)=O(N^2)$. As a result, solving the dynamic recursion \eqref{eq:recursion-DP} backward by moving from the tree leaves toward the root, the number of computations to find the optimal online policy using the DP recursion \eqref{eq:recursion-DP} is at most $O(N\times N^2)$.
\end{proof}

Finally, in a recent work \cite{bayraktar2020bound}, the authors have analyzed the DP \eqref{eq:recursion-DP} in more detail by providing upper and lower bounds on the optimal value function using the approximated dynamic program's viscosity solution. More precisely, it was shown in \cite{bayraktar2020bound} that for any online policy $\Psi$ for the malicious expert: 
\begin{align}\nonumber
    \limsup_{N\to\infty}\frac{V^{*}_{0}(0.5)}{N} \leq 1 - \mu^2,\ \ \
    \liminf_{N\to\infty}\frac{V^{*}_{0}(0.5)}{N} > 1 - \mu.
\end{align}

\begin{figure}[t]
\begin{center}\vspace{0.5cm}
\includegraphics[totalheight=.2\textheight,
width=.3\textwidth,viewport=75 0 725 650]{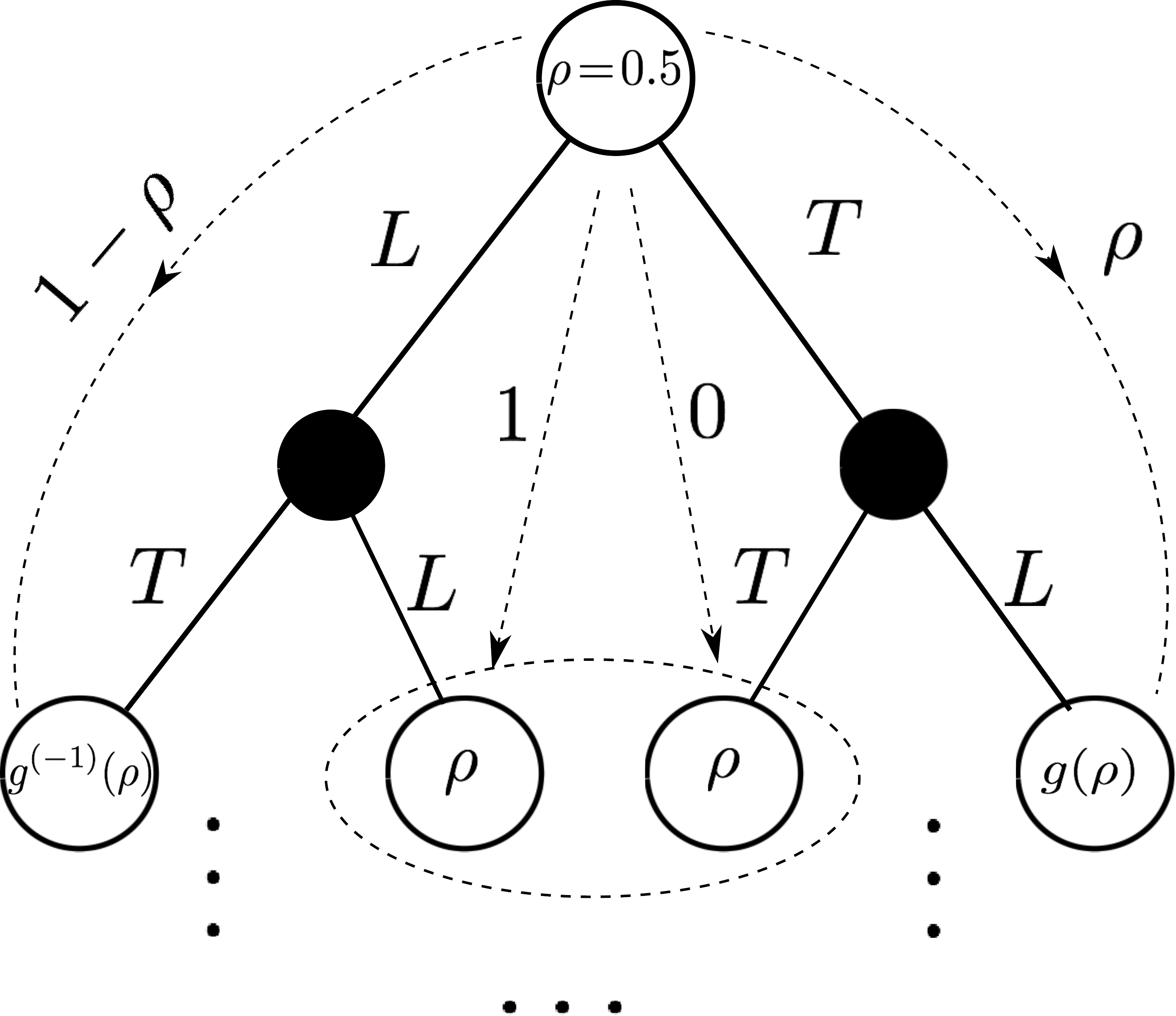}
\end{center}
\caption{Illustration of the first level (root) and the second level of the decision tree. The top actions connecting the root to the intermediate black circles correspond to the honest expert's decisions. The bottom actions connecting the black circles to the second level of the tree correspond to the malicious expert's decisions. Although the second level originally has four nodes, however, two of them can be grouped and be reduced to only three states. The weights on the dashed paths simply denote the loss of the system by following that path.}
\label{fig:tree}
\end{figure}

\subsection{A Generalization to Multiple Experts}
Here we provide a generalization of the problem to the case of many honest experts and one adversary. Without loss of generality, we again assume that the malicious expert is expert $1$ and that all the other experts $i\in\{2,\ldots,K\}$ are honest who make a correct prediction with different probabilities $\mu_i$ (the accuracy of expert $i$). That is,
\begin{align}\nonumber
  	x^i_k=\left\{
		\begin{array}{ll}
		y_k &\text{w.p. } \quad\mu_i,\\
		1 - y_k  &\text{w.p. } \quad 1-\mu_i. 
		\end{array} \right.
\end{align}
We assume that at round $k$, expert $1$ knows the true outcome $y_k$, the accuracy of the honest experts, and the whole history of predictions up to round $k-1$, i.e., $\{\tilde{p}^j_{\ell}, x^{j}_{\ell}, y_{\ell}:  \forall \ell\in[k-1], j\in[K]\}$. Given this information set, the goal of the online malicious expert is to produce a sequence of predictions $\{x^{1}_k\}_{k=1}^{N}$ over a fixed finite horizon $N$ in order to maximize the expected accumulated loss of the system,
\begin{align}\nonumber
\max\limits_{x^1_1,\ldots,x^1_K} \mathbb{E}_{x^i_k, k\in[N], i\neq 1}\big[\sum\limits_{k=1}^N l(\hat{y}_k,y_k)\big], 
\end{align}
where the expectation is taken over the randomization of all the honest experts' predictions $\{x^i_k, k\in[N], i\neq 1\}$. We let $\small{\vec{p}_k = (p^1_k,p^2_k,...,p^K_k)}$ be the {\textit{state}} or {\textit{weight}} vector of all experts at round $k$ and $\small{\vec{\tilde{p}}_k}=(\tilde{p}^1_{k},\ldots,\tilde{p}^K_{k})$ be the corresponding normalized weight vector where $\tilde{p}^i_{k}=\frac{p^i_{k}}{\sum_{i\in [K]} p^i_{k}},\  i\in[K]$ is the normalized weight of expert $i$. Note that we always have $\sum_{i\in [K]} \tilde{p}^i_{k} = 1, \forall k$. Moreover, we let $\phi_{x_k^1}(\vec{p}_{k})$ be the state transition at stage $k$ given that the online malicious expert takes the action $x_k^1\in\{0,1\}$, i.e., 
\begin{align}\nonumber
\phi_{x_k^1}(\vec{p}_k)=\big(p^1_{k+1} ,\ldots,p^K_{k+1}\big),
\end{align}
where $p^i_{k+1} = p^i_k \epsilon^{l(x^i_k,y_k)}$. In addition, the learning algorithm predicts $\hat{y}_k$ at time $k$ by 
\begin{align}\nonumber
  	\small{\hat{y}_{k}
  	=\frac{\sum_{i\in [K]} p^i_{k}x^i_{k}}{\sum_{i\in [K]} p^i_{k}}}=\sum_{i\in[K]}\tilde{p}^i_kx^i_k. 
\end{align}
We note that for $K=2$, this generalized setting coincides with that given in Section \ref{sec:model}. Although characterizing the structure of the optimal online policy in the generalized setting can be very complicated, in the next section we provide some numerical experiments to study the behavior of the optimal policy with multiple honest experts.

\section{Simulations} \label{sec:simulation}

Performance of the false, ratio, and optimal online policies has been simulated numerically and compared in Figure \ref{fig:performance}. In addition to offline and online policies, the case of no adversary with two identical honest experts has also been simulated in order to show the effect of an adversary in the system. 

As can be seen in all plots, an adversary simply adopting the false policy incurs extra loss compared to a system where a malicious expert is not present. The ratio policy imposes more loss than the false policy when the number of stages is large. The optimal online policy imposes a strictly greater loss than the optimal offline policy. Moreover, as the number of stages increases, the gap between the loss of the optimal online policy and either of the offline policies also increases. Similarly, the gain of the bonus term $B$, which is the difference between the curves of the false policy and the ratio policy, increases as the number of stages increases. In fact, it can be observed that the ratio policy closely mimics the structure of the optimal offline policy. For instance, in the middle subfigures of Figure \ref{fig:performance}, the optimal offline expected loss for several values of $N=10,12,14,16,18,20,22,24,26$ are plotted using black dots. As these values are very close to the expected loss of the ratio policy and even coincide in certain cases (e.g., $N=10,14,16$), we believe that the optimal offline policy for finite $N$ belongs to the class of ratio policies, given that one could properly round the block lengths using the CDF of normal distribution.
\begin{figure*}[t]
\begin{center}
\vspace{-0.75cm}
\includegraphics[width=\textwidth]{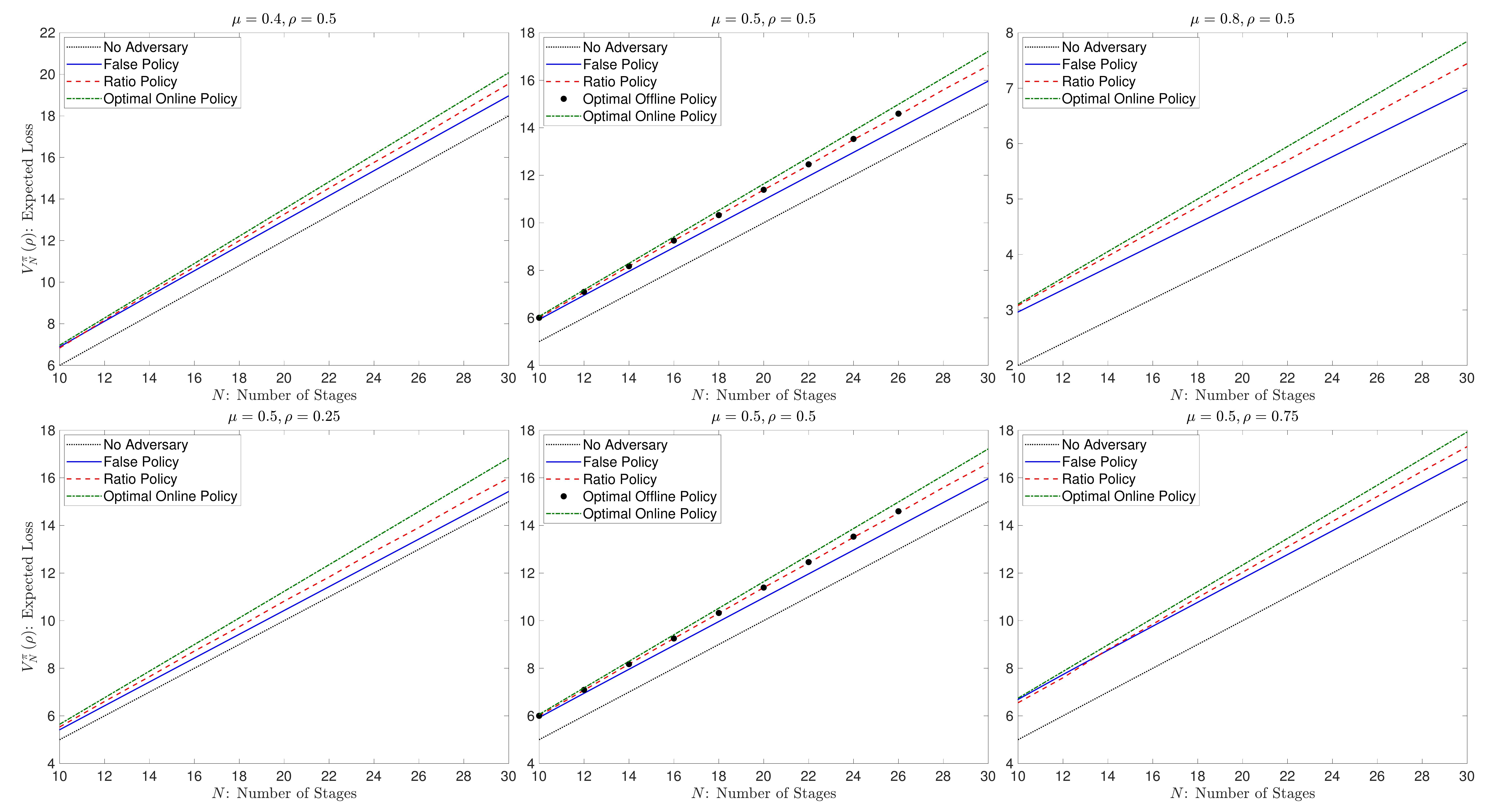}
\vspace{-0.5cm}
\caption{Performance comparison of the false, ratio and optimal online policies for different accuracies and initial relative weights. In all figures we set $\epsilon=\frac{1}{e}$. }\label{fig:performance}
\end{center}
\end{figure*}

Finally, using the generalization to multiple experts in Section \ref{sec:generalization}, a system with four honest experts and one adversary is simulated and compared to a system with one honest expert and one malicious expert. In the 5-expert model, all experts have identical initial weights. In the 2-expert model, the adversary's initial relative weight is $\rho=0.2$, which is the same as that in the 5-expert model. The accuracy of the honest expert in the 2-expert model is the mean of the four honest experts' accuracy in the 5-expert model. Two cases are considered: the homogeneous case (accuracies of all honest experts are identical, $\mu_2=\mu_3=\mu_2=\mu_5=0.5$) and the heterogeneous case (accuracies of honest experts are distinct, $\mu_2=0.3,\mu_3=0.4,\mu_4=0.6,\mu_5=0.7$). The mean of the accuracies of honest experts is the same for the two cases. Expected loss for the 5-expert model is estimated as follows. In each play, a sequence of actions $\{0,1\}^{N}$ is randomly generated for each honest expert $i \in \{2,3,4,5\}$ according to his accuracy $\mu_i$, and the adversary chooses his optimal policy against the honest experts' strategies. This process is repeated 100 times, and the expected loss is approximated by the empirical mean of the losses for all the 100 plays. 

Numerical results are shown in Figure \ref{fig:multi_experts}. As the curve for the homogeneous 5-expert model is very close to that for the 2-expert model, it suggests that the 2-expert model can well approximate the system with multiple homogeneous honest experts by replacing all honest experts with a single one of combined relative weight and the same accuracy. The difference between the curves for the 2-expert model and the heterogeneous 5-expert model is greater, probably because the optimal online policy in the generalized heterogeneous setting is difficult to be approximated by only 2-experts. 
\begin{figure}[t]
\begin{center}
\includegraphics[width=.5\textwidth]{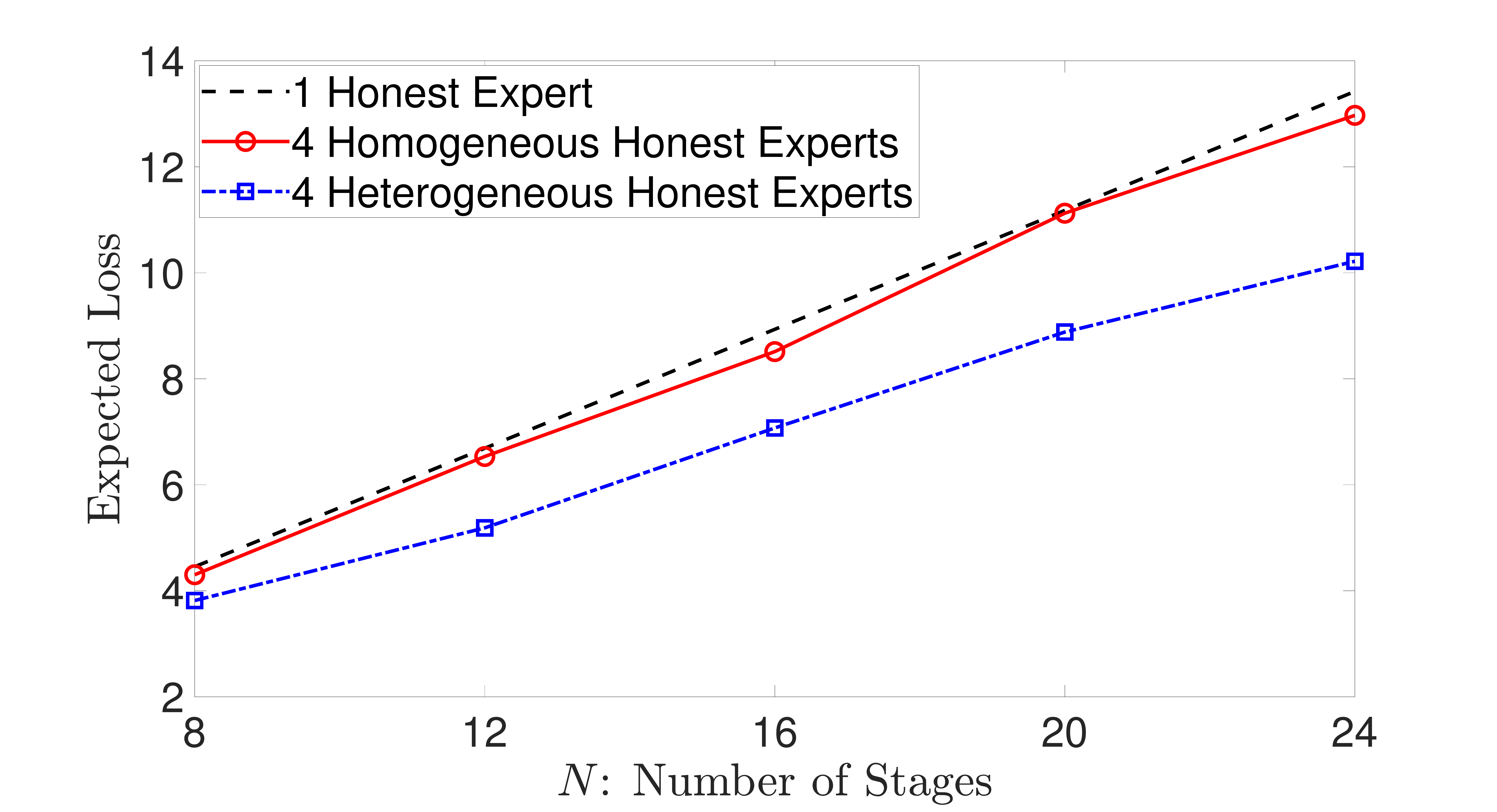}
\end{center}
\caption{Comparison of the online policies for 2-expert and 5-expert models. In the 2-expert model, $\mu=0.5, \rho=0.2$. In the homogeneous 5-expert model, $\mu_2=\mu_3=\mu_2=\mu_5=0.5, \rho=0.2$; in the heterogeneous 5-expert model, $\mu_2=0.3,\mu_3=0.4,\mu_4=0.6,\mu_5=0.7, \rho=0.2$. $\epsilon=\frac{1}{e}$. }
\label{fig:multi_experts}
\end{figure}

\section{Conclusions and Future Directions}\label{sec:conclusion}

In this paper, we considered an adversarial learning system with two experts of whom one is malicious. The malicious expert aims to impose the maximum loss on the system by strategically reporting false predictions. We analyzed the optimal policy for the malicious expert under both offline and online settings. In the offline setting, we showed that finding the optimal policy for the adversary is essentially a hard discrete optimization problem whose solution can be approximated within a negligible (sub-linear) additive term. In particular, we provided a more refined policy that closely mimics the behavior of the optimal offline policy.  We then considered the optimal policy for the online malicious expert and showed that it can be efficiently computed using a dynamic program, and generalized the setting to multiple experts.  

This work opens many interesting directions for future research. It would be interesting to see whether the structure of the optimal policy for the online adversary can be characterized in a closed-form. One possible direction is to leverage the dynamic recursion \eqref{eq:recursion-DP} to show that the optimal value function possesses some nice properties such as convexity, which allows us to prove a class of optimal threshold policies for the online adversary \cite{bertsekas1995dynamic}. Another interesting direction is to use learning schemes other than the MW algorithm (e.g., upper confidence bound algorithm) as the underlying learning scheme and study their robustness against adversarial attacks. When there are many experts in the system can we use mean-field approximation as an effective tool to approximate the optimal policies? Finally, it is interesting to study the game-theoretic version of this work in which the MW learning system can be strategic and not only penalizes the malicious expert but also detects and eliminates it from the system.

\bibliographystyle{IEEEtran}
\bibliography{example_paper}

\appendices

\section{Auxiliary Lemmas}\label{ap-static-preliminary-proofs}

\begin{proposition}\label{prop-no-inf}
The optimal online strategy for an adversary with no information about an arbitrary sequence of true outcomes $\{y_k\}$ is to choose $x_k^1\in\{0,1\}$ with probability $\frac{1}{2}$ and independently for every $k\in [N]$.
\end{proposition} 
\begin{proof}
   Let us fix an arbitrary stage $k$. The expected loss incurred in stage $k$ is the conditional expectation given the entire history of outcomes and predictions up to stage $k-1$ taken over the past and current actions of the honest expert: 
    \begin{align}\nonumber
        \mathbb{E}_{x^2_1,\dots,x^2_k}[l(\hat{y}_k,y_k)] 
        \equiv \mathbb{E}_{x^2_1,\dots,x^2_k}[l(\hat{y}_k,y_k)|\{\tilde{p}^1_l,x^1_l,x^2_l,y_l\}_{l=1}^{k-1}]. 
    \end{align}
As $y_k$ can be chosen arbitrarily, and the honest expert's prediction at stage $k$ is independent of the previous stages, the history of predictions up to stage $k-1$ cannot give any information to the adversary about $y_k$. As a result, we have 
    \begin{align}\label{eq:uncondexpec}
        \mathbb{E}_{x^2_1,\dots,x^2_k}[l(\hat{y}_k,y_k)]= \mathbb{E}_{x^2_k}[l(\hat{y}_k,y_k)].
    \end{align}
Therefore, in this case, the adversary becomes memoryless and treats every stage as a new restart. Now for the absolute loss function $l(\hat{y},y)=|\hat{y}-y|$, one can compute the expected loss in a closed form for stage $k$ using \eqref{eq:uncondexpec}. Let us assume that the adversary chooses $x^1_k=0$ with probability $q$ and chooses $x^1_k=1$ with probability $1-q$. Depending on the true outcome $y_k$, the expected loss equals one of the following terms: 
    \begin{align}\nonumber
        &\mathbb{E}_{x^2_k}[l(\hat{y}_k,y_k)|y_k=0]= 1-\mu+\mu\rho-q\rho, \cr 
        &\mathbb{E}_{x^2_k}[l(\hat{y}_k,y_k)|y_k=1]= 1-\mu+\mu\rho-(1-q)\rho, 
    \end{align}
where $\rho$ denotes the relative weight of the adversary at the beginning of stage $k$. Since the adversary has no information about whether $y_k=0$ or $y_k=1$, it must choose $q$ to maximize the minimum of the above two expressions. Now an easy calculation shows that for $q=\frac{1}{2}$, the above two equations coincide, which shows that predicting with probability $\frac{1}{2}$ at each stage is the optimal online strategy.
\end{proof}

\begin{lemma}\label{lemm:appendix}
Let $f(r):=r-\ln(1+ae^r), h(r):=\ln(a+e^r)$,
\begin{align}\nonumber
&\epsilon(r):=f(r+1)-f(r)-g^{(r)}(\rho),\cr 
&\delta(r):=h(r+1)-h(r)-g^{(-r)}(\rho).
\end{align}
Then for any $r\ge 0$ we have, 
\begin{align}\nonumber
& 0 \geq \epsilon(r)\geq \frac{1}{1+ae^{r+1}}-\frac{1}{1+ae^{r}}, \cr 
&0\leq \delta(r)\leq \frac{1}{1+ae^{-(r+1)}}-\frac{1}{1+ae^{-r}},
\end{align}
where we recall that $a=\frac{1}{\rho}-1$, for some $\rho\in (0,1)$.
\end{lemma}
\begin{proof}
Since $\frac{d}{dr}\epsilon(r)=\frac{ae^{r}(ae^r-e+2)}{(1+ae^r)^2(1+ae^{r+1})}$ has only one root given by $e^r=\frac{e-2}{a}$, by evaluating $\epsilon(r)$ in the root of its derivative as well as the boundary of its domain we get,
\begin{align}\nonumber
\epsilon(r)=\begin{cases} \frac{e-2}{e-1}-\ln(e-1)<0 & \mbox{if } \  e^r=\frac{e-2}{a}, \\
1+\ln(\frac{1+a}{1+ae})-\frac{1}{a+1}\leq 0  & \mbox{if}, \  r=0, \\ 
0 & \mbox{if } \  \substack{a=0, \ \mbox{or} \ r\to \infty, \\ \mbox{or} \ a\to \infty}. \end{cases}
\end{align}
This shows that $\epsilon(r)\leq 0,\ \forall r, a\in [0,\infty)$. On the other hand, for every $r>0$, using the Mean-value Theorem we have $f(r+1)-f(r)=f'(\eta_r)$, for some $\eta_r\in [r, r+1]$. Since $f'(\eta_r)=\frac{1}{1+ae^{\eta_r}}>\frac{1}{1+ae^{r+1}}$, using \eqref{eq:composition-rule} we can write,
\begin{align}\nonumber
\epsilon(r)&=f(r+1)-f(r)-g^{(r)}(\rho)\cr 
&\ge \frac{1}{1+ae^{r+1}}-g^{(r)}(\rho)\cr 
&=\frac{1}{1+ae^{r+1}}-\frac{1}{1+ae^{r}}.
\end{align}

Similarly, $\frac{d}{dr}\delta(r)=-\frac{ae^r(e^r-(e-2)a)}{(a+e^r)^2(a+e^{r+1})}$, which has only one root at $e^r=(e-2)a$. Therefore, by evaluating $\delta(r)$ in the root of its derivative as well as the boundary points one can easily see that $\delta(r)\ge 0$. Again, using the Mean-value Theorem, there exists $\zeta_r\in [r,r+1]$ such that $h(r+1)-h(r)=h'(\zeta_r)=\frac{1}{1+ae^{-\zeta_r}}\leq \frac{1}{1+ae^{-(r+1)}}$. This shows that, 
\begin{align}\nonumber
\delta(r)&\leq \frac{1}{1+ae^{-(r+1)}}-g^{(-r)}(\rho)\cr 
&=\frac{1}{1+ae^{-(r+1)}}-\frac{1}{1+ae^{-r}}.
\end{align}
\end{proof}

\begin{lemma}\label{lemm:log-bound}
Let $m_1,m_2,\ldots,m_k$ be positive integers and define $M_{\ell}:=\sum_{j=1}^{\ell}m_j, \ \ell\in[k]$, (by convention we let $M_0=0$). Then $\sum_{\ell=1}^{k}\frac{m_{\ell}}{\sqrt{M_{\ell}}}= O(\sqrt{M_k\ln M_k})$.
\end{lemma}
\begin{proof}
Starting from the left-hand side and using Cauchy-Schwarz inequality, we have,
\begin{align}\label{eq:cauchy}
\sum_{\ell=1}^{k}\frac{m_{\ell}}{\sqrt{M_{\ell}}}&=\sum_{\ell=1}^{k}\sqrt{m_{\ell}}\times \sqrt{\frac{m_{\ell}}{M_{\ell}}}\leq \sqrt{M_{k}}\sqrt{\sum_{\ell=1}^{k} \frac{m_{\ell}}{M_{\ell}}}.
\end{align}
Next for every $\ell$, we can write 
\begin{align}\nonumber
\frac{m_{\ell}}{M_{\ell}}&=\underbrace{\frac{1}{m_{\ell}\!+\!M_{\ell-1}}\!+\!\frac{1}{m_{\ell}\!+\!M_{\ell-1}}\!+\!\ldots\!+\!\frac{1}{m_{\ell}\!+\!M_{\ell-1}}}_\text{$m_{\ell}$ times}\cr 
&\leq \frac{1}{1\!+\!M_{\ell-1}}\!+\!\frac{1}{2\!+\!M_{\ell-1}}\!+\!\ldots\!+\!\frac{1}{m_{\ell}\!+\!M_{\ell-1}}.
\end{align}
Summing the above relation for all $\ell=1,\ldots,k$, we get 
\begin{align}\nonumber
\sum_{\ell=1}^{k} \frac{m_{\ell}}{M_{\ell}}&\leq \sum_{\ell=1}^{k} \frac{1}{1\!+\!M_{\ell-1}}\!+\!\frac{1}{2\!+\!M_{\ell-1}}\!+\!\ldots\!+\!\frac{1}{m_{\ell}\!+\!M_{\ell-1}}\cr
&=\sum_{j=1}^{M_k}\frac{1}{j}\leq 1+\ln M_{k}.
\end{align}
Using this relation into \eqref{eq:cauchy} we get the desired bound.
\end{proof}

\begin{lemma}[{\bf Berry-Esseen Theorem \cite{berry1941accuracy}}]\label{lemm:berry}
Let $V_i$ be independent random variables with mean $a_i$ and variance $s^2_i$, and define $S_t:=\sum_{i=1}^{t}V_i$. Then there exists an absolute constant $c_0$ such that for all $t$ the CDF of $S_t$, denoted by $F_t(x)$, satisfies 
\begin{align}\nonumber
\sup_x\Big|F_t(x)-\Phi\Big(\frac{x-\sum_{i=1}^{t}a_i}{\sqrt{\sum_{i=1}^{t}s^2_i}}\Big)\Big|\leq \frac{c_0\max_i\{\frac{\mathbb{E}|X_i-a_i|^3}{s^2_i}\}}{\sqrt{\sum_{i=1}^{t}s^2_i}} .
\end{align}
\end{lemma}

\begin{lemma}\label{lemm:berry-essen}
Let $X\sim Bin(n,\mu)$ and $Y\sim Bin(m,(1-\mu))$ be two independent Binomial distributions. Then, there exists a constant $c$ such that $\left|\mathbb{E}\big[\frac{1}{1+e^{X-Y}}\big]-\Phi(-\frac{\nu}{\sigma})\right|\leq \frac{c}{\sigma}$, where $\nu=n\mu-(1-\mu)m$, $\sigma^2=\mu(1-\mu)(n+m)$, and $\Phi(\cdot)$ is the CDF of the standard normal distribution.
\end{lemma}
\begin{proof}
Let $p(\gamma)$ and $F(\gamma)$ denote the pmf and CDF of the random variable $X-Y$, respectively. Then,
\begin{align}\label{eq:E-F}
\!\!\!\!\Big|F(0)\!-\!\mathbb{E}[\frac{1}{1\!+\!e^{X-Y}}]\Big|&=\Big|\mathbb{P}(X\!-\!Y\leq 0)\!-\!\mathbb{E}[\frac{1}{1\!+\!e^{X-Y}}]\Big|\cr 
&=\Big|\sum_{i=-\infty}^{0}p(i)-\sum_{i=-\infty}^{\infty}\frac{p(i)}{1+e^i}\Big|\cr 
&=\Big|\sum_{i=0}^{\infty}\frac{p(-i)}{1+e^{i}}-\sum_{i=1}^{\infty}\frac{p(i)}{1+e^i}\Big|\cr 
&\leq \sum_{i=0}^{\infty}\frac{p(-i)}{1+e^i}+\sum_{i=0}^{\infty}\frac{p(i)}{1+e^i}\cr 
&\leq  \sum_{i=0}^{\infty}p(-i)e^{-i}+\sum_{i=0}^{\infty}p(i)e^{-i}\cr 
&\leq 2p_{\max}\sum_{i=0}^{\infty}e^{-i}=\frac{2e}{e-1}p_{\max},
\end{align}
where $p_{\max}=\max_{i}\{p(i)\}$. Next, using Berry-Esseen Theorem (Lemma \ref{lemm:berry}), and noting that $X$ and $Y$ can be written as sum of $n$ and $m$ independent Bernoulli random variables $Ber(\mu)$ and $Ber(1-\mu)$, respectively, we get
\begin{align}\label{eq:berry}
\sup_{\gamma}\Big|F(\gamma)-\Phi\Big(\frac{\gamma-\nu}{\sigma}\Big)\Big|\leq \frac{c_0(\mu^2+(1-\mu)^2)}{\sigma}. 
\end{align}
Now for every $i$ we can write,
\begin{align}\label{eq:p_max}
p(i)&=F(i)-F(i-1)\cr 
&\leq \Phi(\frac{i-\nu}{\sigma})-\Phi(\frac{i-1-\nu}{\sigma})+\frac{2c_0(\mu^2+(1-\mu)^2)}{\sigma}\cr 
&\leq \Phi'(\eta)\times\left(\frac{i-\nu}{\sigma}-\frac{i-1-\nu}{\sigma}\right)+\frac{2c_0(\mu^2+(1-\mu)^2)}{\sigma}\cr 
&=\frac{1}{\sqrt{2\pi}}e^{-\frac{\eta^2}{2}}\times \frac{1}{\sigma}+\frac{2c_0(\mu^2+(1-\mu)^2)}{\sigma}\cr 
&\leq \frac{1}{\sqrt{2\pi}\sigma}+\frac{2c_0(\mu^2+(1-\mu)^2)}{\sigma}=\frac{c_1}{\sigma}
\end{align}
where $c_1:=\frac{1}{\sqrt{2\pi}}+2c_0(\mu^2+(1-\mu)^2)$. The first inequality is due to \eqref{eq:berry} and the second inequality is by Mean-value Theorem for some $\eta\in [\frac{i-1-\nu}{\sigma}, \frac{i-\nu}{\sigma}]$. As a result, $p_{\max}\leq \frac{c_1}{\sigma}$. Substituting \eqref{eq:p_max} into \eqref{eq:E-F}, we have
\begin{align}\label{eq:F0}
\Big|F(0)-\mathbb{E}[\frac{1}{1+e^{X-Y}}]\Big|\leq \frac{2ec_1}{(e-1)\sigma}.
\end{align}
Finally, adding \eqref{eq:F0} with \eqref{eq:berry} when $\gamma=0$, and using the triangle inequality, we get
\begin{align}\nonumber
\Big|\mathbb{E}[\frac{1}{1\!+\!e^{X-Y}}]\!-\!\Phi(-\frac{\nu}{\sigma})\Big|\!\leq\! \frac{2ec_1}{(e\!-\!1)\sigma}\!+\!\frac{c_0(\mu^2\!+\!(1\!-\!\mu)^2)}{\sigma}\!=\!\frac{c}{\sigma}, 
\end{align}where $c:=\frac{2ec_1}{e-1}+c_0(\mu^2+(1-\mu)^2)$ is a positive constant.
\end{proof}

\begin{lemma}\label{lemm:main}
Let $X_{\ell}\sim Bin(N_{\ell},\mu), Y_{\ell}\sim Bin(M_{\ell},1-\mu)$, $\ell\in[k]$, be mutually independent (i.e., for every $i\neq j$, $X_i$ and $Y_j$ are independent). Moreover, assume $N_{\ell}=\sum_{i=1}^{\ell}n_{i}$ and $M_{\ell}=\sum_{i=1}^{\ell}m_{i}$, where $n_i, m_i\in\mathbb{Z}^+$, and $N_{k}+M_{k}=N$. Then $B:=\sum_{\ell=1}^{k}\mathbb{E}\Big[\frac{1}{1+e^{X_{\ell}-Y_{\ell}}}-\frac{1}{1+e^{X_{\ell}-Y_{\ell-1}}}\Big]=O(\sqrt{N\ln N})$.
\end{lemma}
\begin{proof}
For any $\ell=1,\ldots,k$, define
\begin{align}\label{eq:mean-variance}
&\mu_{2\ell}\!:=N_{\ell}\mu\!-\!M_{\ell}(1\!-\!\mu), \ \ \ \ \ \ \ \ \ \  \sigma^2_{2\ell}:=\mu(1\!-\!\mu)(N_{\ell}\!+\!M_{\ell}),\cr 
&\mu_{2\ell\!-\!1}\!:=N_{\ell}\mu\!-\!M_{\ell\!-\!1}(1\!-\!\mu), \ \ \ \  \sigma^2_{2\ell\!-\!1}\!:=\mu(1\!-\!\mu)(N_{\ell}\!+\!M_{\ell\!-\!1}), 
\end{align}where $M_{\ell}=\sum_{i=1}^{\ell}m_{i}$ and $N_{\ell}=\sum_{i=1}^{\ell}n_{i}$ (recall that $m_i$ and $n_i$ denote, respectively, the lengths of the $i$th true and false blocks in an offline policy $\Psi$). Using Lemma \ref{lemm:berry-essen} we can write 
\begin{align}\label{eq:integral-apx}
B&\stackrel{\text{(a)}}{\leq}2\sum_{\ell=1}^{k}\left[\Phi(-\frac{\mu_{2\ell}}{\sigma_{2\ell}})-\Phi(-\frac{\mu_{2\ell-1}}{\sigma_{2\ell-1}})+c\left(\frac{1}{\sigma_{2\ell}}+\frac{1}{\sigma_{2\ell-1}}\right)\right]\cr  
&\stackrel{\text{(b)}}{\leq}\! 2\!\sum_{\ell=1}^{k}\!\Big[\Phi(-\frac{\mu_{2\ell}}{\sigma_{2\ell}})-\Phi(-\frac{\mu_{2\ell-1}}{\sigma_{2\ell-1}})\Big]\cr 
&\qquad+\frac{2c}{\sqrt{\mu(1-\mu)}}\sum_{\ell=1}^{k}\left(\!\frac{1}{\sqrt{2\ell}}+\frac{1}{\sqrt{2\ell-1}}\right)\cr
&\leq 2\sum_{\ell=1}^{k}\Big[\Phi(-\frac{\mu_{2\ell}}{\sigma_{2\ell}})-\Phi(-\frac{\mu_{2\ell-1}}{\sigma_{2\ell-1}})\Big]+\frac{2c}{\sqrt{\mu(1\!-\!\mu)}}\int_{0}^{2k}\!\!\!\frac{1}{\sqrt{x}}dx\cr  
&\stackrel{\text{(c)}}{\leq} 2\sum_{\ell=1}^{k}\Big[\Phi(-\frac{\mu_{2\ell}}{\sigma_{2\ell}})-\Phi(-\frac{\mu_{2\ell-1}}{\sigma_{2\ell-1}})\Big]+4c\sqrt{\frac{N}{\mu(1-\mu)}},
\end{align}
where $(a)$ is due to Lemma \ref{lemm:berry-essen}, and $(b)$ holds because $N_{\ell}+M_{\ell}=\sum_{i=1}^{\ell}(n_i+m_i)$ is the sum of $2\ell$ positive integers, and thus $\sigma_{2\ell}\ge \sqrt{2\mu(1-\mu)\ell}$ (similarly $\sigma_{2\ell-1}\ge \sqrt{\mu(1-\mu)(2\ell-1)}$). Finally $(c)$ holds because $2k\leq N$.  

We proceed by showing a sub-linear upper bound on $\sum_{\ell=1}^{k}\big[\Phi(-\frac{\mu_{2\ell}}{\sigma_{2\ell}})-\Phi(-\frac{\mu_{2\ell-1}}{\sigma_{2\ell-1}})\big]$. Let $\beta$ be a constant defined by $\beta:=\frac{1}{\sqrt{2\pi\mu(1-\mu)}}$. Using the Mean-Value Theorem and since $\Phi'(x)=\frac{1}{\sqrt{2\pi}}e^{-\frac{x^2}{2}}\leq \frac{1}{\sqrt{2\pi}}, \forall x$, we can write,
\begin{align}\nonumber
&\sum_{\ell=1}^{k}\Big[\Phi(-\frac{\mu_{2\ell}}{\sigma_{2\ell}})-\Phi(-\frac{\mu_{2\ell-1}}{\sigma_{2\ell-1}})\Big]\leq \sum_{\ell=1}^{k}\frac{1}{\sqrt{2\pi}}\left(\frac{\mu_{2\ell-1}}{\sigma_{2\ell-1}}-\frac{\mu_{2\ell}}{\sigma_{2\ell}}\right)\cr
&\qquad=\beta\sum_{\ell=1}^{k}\left(\frac{N_{\ell}\mu-M_{\ell-1}(1-\mu)}{\sqrt{N_{\ell}+M_{\ell-1}}}-\frac{N_{\ell}\mu-M_{\ell}(1-\mu)}{\sqrt{N_{\ell}+M_{\ell}}}\right)\cr
&\qquad\leq \beta\sum_{\ell=1}^{k}\left(\frac{-M_{\ell-1}}{\sqrt{N_{\ell}+M_{\ell-1}}}+\frac{M_{\ell}}{\sqrt{N_{\ell}+M_{\ell}}}\right)\cr 
&\qquad\leq  \beta\sum_{\ell=1}^{k}\frac{M_{\ell}-M_{\ell-1}}{\sqrt{N_{\ell}+M_{\ell}}}=\beta\sum_{\ell=1}^{k}\frac{m_{\ell}}{\sqrt{N_{\ell}+M_{\ell}}}\cr 
&\qquad\leq \beta\sum_{\ell=1}^{k}\frac{m_{\ell}}{\sqrt{M_{\ell}}}.
\end{align}
Finally, using Lemma \ref{lemm:log-bound} and noting that $\sum_{\ell=1}^{k}m_{\ell}\leq N$, we obtain 
\begin{align}\label{eq:bonus-log-bound}
\sum_{\ell=1}^{k}\Big[\Phi(-\frac{\mu_{2\ell}}{\sigma_{2\ell}})-\Phi(-\frac{\mu_{2\ell-1}}{\sigma_{2\ell-1}})\Big]=O(\sqrt{N\ln N}).
\end{align}
This together with \eqref{eq:integral-apx} completes the proof.
\end{proof}

\ifCLASSOPTIONcaptionsoff
  \newpage
\fi

\end{document}